\documentclass{article}

\usepackage{arxiv}

\usepackage[utf8]{inputenc} 
\usepackage[T1]{fontenc}    
\usepackage{hyperref}       
\usepackage{url}            
\usepackage{booktabs}       
\usepackage{amsfonts}       
\usepackage{microtype}      
\usepackage{lipsum}
\usepackage{graphicx}
\usepackage{amsmath}
\usepackage{amsthm}
\usepackage{inputenc}
\usepackage{xcolor}
\usepackage{wrapfig}
\usepackage[normalem]{ulem}
\usepackage{authblk}
\newtheorem{theorem}{Theorem}
\theoremstyle{remark}

\newtheorem{lemma}{Lemma}

\DeclareMathOperator*{\argmax}{argmax} 

\title{Majority Voting and the Condorcet's Jury Theorem}
\author[1,2]{Hanan Shteingart}
\author[2,3]{Eran Marom}
\author[1]{Igor Itkin}
\author[1]{Gil Shabat}
\author[1]{\\Michael Kolomenkin}
\author[1,4]{Moshe Salhov}
\author[5]{Liran Katzir}
\affil[1]{Playtika Brain}
\affil[2]{Yandex School of Data Science}
\affil[3]{Gong.io}
\affil[4]{School of Computer Science, Tel Aviv University, Israel}
\affil[5]{Technion – Israel Institute of Technology}

\begin{document}
\maketitle

\begin{abstract}
There is a striking relationship between a three hundred years old Political Science theorem named "Condorcet's jury theorem" (1785), which states that majorities are more likely to choose correctly when individual votes are often correct and independent, and a modern Machine Learning concept called "Strength of Weak Learnability" (1990), which describes a method for converting a weak learning algorithm into one that achieves arbitrarily high accuracy and stands in the basis of Ensemble Learning. Albeit the intuitive statement of Condorcet's theorem, we could not find a compact and simple rigorous mathematical proof of the theorem neither in classical handbooks of Machine Learning nor in published papers. By all means we do not claim to discover or reinvent a theory nor a result. We humbly want to offer a more publicly available simple derivation of the theorem. We will find joy in seeing more teachers of introduction-to-machine-learning courses use the proof we provide here as an exercise to explain the motivation of ensemble learning.
\end{abstract}

\keywords{Majority Voting \and Condorcet's Jury Theorem \and Ensemble Learning \and }
\section{Introduction}
The idea of combining several opinion is ancient and has been formalized with the "Condorcet's jury theorem" (1785)  \cite{condorcet1785essay},  which is considered by some to be the theoretical basis for democracy \cite{ladha1992condorcet}. The theorem states that a majority of independent individuals who make correct decisions with probability $p>1/2$, i.e. better than by random choice, are more likely to choose correctly than each individual. The theorem can be separated into two parts. The first, non-asymptotic monotonous part, states that the probability of the majority vote to be correct increases monotonically with the size of the population ($n$). The second, asymptotic part, states that the probability to be correct converges to 1 as the number of individuals increases to infinity.

Analogically, in Machine Learning, the problem of combining output of several simple classifiers culminated in Ensemble Learning (e.g. AdaBoost \cite{freund1995desicion}), which is often considered to be one of the most powerful learning ideas introduced in the last twenty years \cite{friedman2001elements}. Ensemble Learning is a machine learning paradigm where multiple classifiers are trained to solve the same problem. The final decision is obtained by taking a (weighted) vote of their predictions \cite{dietterich2000ensemble}. The combination of different classifiers can be implemented using a variety of strategies, among which majority vote is by far the simplest \cite{kittler1998combining}. Albeit its simplicity, if the errors among the classifiers are not correlated, it has been claimed that the majority vote is the best strategy \cite{kittler2003multiple} (see also Theorem \ref{th:optimal} below). Notably, the suggestion to combine weak performing units into a better one was previously studied in the context of hardware design by giants like Shannon and Moore \cite{moore1956reliable}.

Notwithstanding, albeit the simplicity and usefulness of the majority decision as a rule, as presented above and as reflected in the Condorcet's jury theorem, it is surprisingly very hard to find a rigours, compact and simple derivation of it in the professional literature. For example, one paper which studied the majority rule extensively states that the recursive formula, derived in this paper, is an "unpublished" result \cite{lam1997application}. Another publications claim that the theorem's proof is ``straightforward'', yet neither provide such a proof, nor cite where such proof can be found \cite{ben2000nonasymptotic, berend2005monotonicity}. Beyond the professional literature, the proof in Wikipedia \cite{wiki} seems in-cohesive as can be demonstrated by a question raised in the Stack Exchange website: "here is a proof in Wikipedia but I am don't understand it's correctness" \cite{stackexchange}. By contrast, in this paper we provide for the first time, to the best of our knowledge, a simple and compact derivation of the Condorcet theorem monotonous and asymptotic parts. 

\begin{figure}
\centering
  \includegraphics[width=0.33\textwidth]{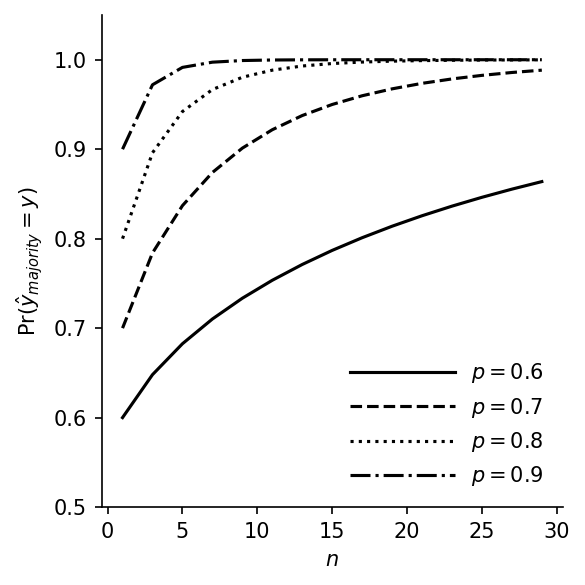}
  \caption{Probability of majority to be correct as a function of number of  classifiers $n$ (odd), for different probabilities of a single classifier to be correct $p$}
  \label{fig:boat1}
\end{figure}

\section{Proof of Condorcet's Jury Theorem}
In Section \ref{sec:optimality} we will show that majority rule (for equiprobable classes) coincides with the Bayes classifier, that is the maximum a posterior decision rule, which minimizes prediction errors. Next, in Section \ref{sec:monotonic} we will prove that the correctness probability of majority rule is monotonically increasing (Theorem \ref{th:cond_monotonic}), as depicted in Figure \ref{fig:boat1}. Finally, in Section \ref{sec:asymptotic} we will show also that this probability converges to one (Theorem \ref{th:con_asymptotic}). The two last theorem will conclude the proof of the Condorcet's Jury Theorem.

\subsection{Optimality of the Majority Rule}
\label{sec:optimality}
Before starting the proof of the Condorcet's jury Theorem, let us first consider its motivation as an optimal Bayes classifier. 
\begin{theorem}
\label{th:optimal}
Let there be an odd number $n=2m+1$ of binary independent classifiers, each classifier being correct with probability $\Pr(\hat{y}=y)=p>1/2$, where $y, \hat{y} \in\{-1,1\}$  are the true and predicted labels, respectively. Moreover, assume that the prior labels are equiprobable $\Pr(y=1)=\Pr(y=0)$. Let $S_n=\sum_{i=1}^{n}{\hat{y}_n}$ be the sum of these $n$ classifier predictions, and let $\hat{y}_{\text{majority}}^n=\text{sign}(S_n)$ be a majority decision rule over the $n$ classifiers. Then, the latter majority rule is a Bayes classifier which minimizes the classification error.
\end{theorem}
\begin{proof}

The Bayes classifier which minimizes the probability of misclassification  $\Pr(\hat{y} \ne y)$ follows the maximum a posterior (MAP) criteria  $\hat{y}_\text{Bayes}=\argmax\limits_y \Pr(y|\{y_i\}_{i=1}^n)$  \cite{devroye2013probabilistic}. We now show that the majority rule is equivalent to MAP.
According to Bayes rule,
\begin{equation}
    Pr(y|\{y_i\}_{i=1}^n)=\frac{\Pr(\{y_i\}_{i=1}^n|y)\Pr(y)}{\Pr(\{y_i\}_{i=1}^n)}
\end{equation}
Therefore, for equally probable classes $\Pr(y=0)=\Pr(y=1)$,  maximizing the MAP is equivalent to maximizing the likelihood: $\argmax\limits_y \Pr(y|\{y_i\}_{i=1}^n)=\argmax\limits_y\Pr(\{y_i\}_{i=1}^n|y) $

The likelihood for $y=1$ follows the Binomial distribution:
\begin{equation}
    \Pr(\{y_i\}_{i=1}^n|y=1)=\prod_{i=1}^n{p^{(y_i+1)/2}(1-p)^{(1-y_i)/2}}
\end{equation}. 
Applying log on both sides, one get the log likelihood as:
\begin{equation}
    log(\Pr(\{y_i\}_{i=1}^n|y=1))=\frac{1}{2}\sum_{i=1}^n{log(p)(y_i+1) + log(1-p)(1-y_i)}
\end{equation}. 
Similarly, for $y=0$, $\log(\Pr(\{y_i\}_{i=1}^n|y=0))=\frac{1}{2}\sum_{i=1}^n{log(1-p)(y_i+1) + log(p)(1-y_i)}$. Therefore, the log likelihood ratio $\textit{LLR}=\log(\frac{\Pr(\{y_i\}_{i=1}^n|y=1))}{\Pr(\{y_i\}_{i=1}^n|y=0))}$  equals to: 
\begin{equation}
    \textit{LLR} =\log(\frac{p}{1-p})\sum_{i=1}^n{y_i}
\end{equation}
The latter concludes the proof as it shows that $\hat{y}_\text{Bayes}=\argmax\limits_y \Pr(y|\{y_i\}_{i=1}^n)=\text{sign}(\sum_{i=1}^n{y_i})$ so that the MAP criteria coincides with the majority rule, which minimize the error probability (for equiprobable prior).
\end{proof}

\subsection{Monotonic Part}
\label{sec:monotonic}
Next, we show the non-asymptotic monotonic part of Condorcet's Jury Theorem which can be rephrased as follows:
\begin{theorem}\label{th:cond_monotonic}
 Let there be an odd number $n=2m+1$ of binary independent classifiers, each classifier being correct with probability $\Pr(\hat{y}=y)=p>1/2$, where $y, \hat{y} \in\{-1,1\}$  are the true and predicted labels, respectively. 
Let $S_n=\sum_{i=1}^{n}{\hat{y}_n}$ be the sum of these $n$ classifier predictions, and let $\hat{y}_{\text{majority}}^n=\text{sign}(S_n)$ be a majority decision rule over the $n$ classifiers. Then, the probability of a majority prediction to be correct $\Pr(\hat{y}_{\text{majority}}^n=y)$ increases with $n$.
\end{theorem}

\begin{proof}
Assume, without loss of generality, that the true label is $y=1$. Then, the probability to have most classifiers (majority) predict the correct label $\Pr(\hat{y}_{\text{majority}}^n=1)$ is  $\Pr(S_n>0)$. Assume two more classifiers are added to the ensemble $n\rightarrow n+2$. Then, the probability to be correct, $\Pr(\hat{y}_{\text{majority}}^{n+2}=1)= \Pr(S_{n+2}>0)$, can be computed by the law of total probability while taking into account all possible events for which this inequality holds.

Note that the since $n$ is odd, the sum $S_n$ must also be odd (so $S_n$ cannot equal 0 or 2, for instance). Also, if $S_n<-1$, no matter what will be the individual predictions of each of the two new added classifiers, it will not be enough to tip the balance and make the overall majority decision prediction correct. We are left with three mutually exclusive events for which the majority decision rule will provide a correct classification: 

\begin{itemize}
\item $S_n>2$, and then the two new classifier predictions do not matter as majority decision correctness is guaranteed; 
\item $S_n=1$, and then at least one of the two new classifier needs to be correct for the majority decision to be correct;
\item $S_n=-1$, and then two new classifiers need to be correct in order to switch the decision of the majority to be correct.
\end{itemize}

The latter three cases can be summarized in the following recursive equation:
\begin{equation}
    \label{eq:eq1}
    \Pr(S_{n+2}>0)=\Pr(S_n>2) + \Pr(S_n=1)(1-(1-p)^2) + \Pr(S_n=-1)p^2
\end{equation}
Note that $\Pr(S_n>2) + \Pr(S_n=1) = \Pr(S_n>0)$ and thus equation Eq. \ref{eq:eq1} can be rewritten as:
\begin{equation}
        \label{eq:eq2}
        \Pr(S_{n+2}>0) - \Pr(S_n>0) = \Pr(S_n=-1)p^2 - \Pr(S_n=1)(1-p)^2
\end{equation}
Therefore, in order to prove Theorem \ref{th:cond_monotonic} one need to show that the series $\Pr(S_n)$ is monotonically increasing, or in other words, that the right hand side of Eq. \ref{eq:eq2} is positive. 

Recall that $S_n$ is binomial, and that $n=2m+1$. Therefore, we know that $\Pr(S_n=1)=\binom{n}{m+1}p^{m+1}(1-p)^m$ and that $\Pr(S_n=-1)=\binom{n}{m}p^m(1-p)^{m+1}$. Also note that for $n=2m+1$ it holds that $\binom{n}{m} = \binom{n}{n-m} = \binom{n}{m+1}$, thus Eq. \ref{eq:eq2} can be rewritten as follows:  
\begin{equation}
    \label{eq:recursive_simplified}
    \Pr(S_{n+2}>0) - \Pr(S_n>0) = \binom{n}{m} p^{m+1}(1-p)^{m+1}(2p-1)
\end{equation}
This completes the proof as the right hand side of Eq. \ref{eq:recursive_simplified} is positive when $p>1/2$.
\end{proof}

\subsection{Asymptotic Part}
\label{sec:asymptotic}

Next, we show that under the same settings as in Theorem \ref{th:cond_monotonic}, when the number of classifiers grows asymptotically to infinity, the majority rule probability to be correct converges to one. We will prove this theorem in two ways. The first way, which is arguably simpler, by using Chebychev's inequality. The second way, by using the recursive formula derived in Theorem \ref{th:cond_monotonic}. 
\begin{theorem}
\label{th:con_asymptotic}
Let $n=2m+1$ be an odd number of binary independent classifiers, each with probability $\Pr(\hat{y}=y)=p>1/2$, where $y, \hat{y}\in\{-1,1\}$  are the true and predicted labels, respectively. Then, the probability of the majority prediction being correct converges to $1$ as $n$ goes to infinity. Formally, 
\begin{equation}
    \lim_{n \rightarrow \infty}\Pr(\hat{y}_{\text{majority}}^n=y)=1
\end{equation}.
\end{theorem}


\subsubsection{Proof Using Chebychev's Inequality}
\label{subsec:chebychev}
\begin{proof}
According to Chebyshev's inequality, if $X$ is a random variable with a finite expected value $\mu$ and a finite non-zero variance $\sigma ^2$. Then for any real number $\epsilon > 0$, the following holds:
\begin{equation}
\label{eq:chebychev}
   \Pr(|X-\mu |<\epsilon ) > 1-{\frac {\sigma ^{2}}{\epsilon ^{2}}}
\end{equation}

Once more, we assume, without loss of generality, that the true label is $y=1$. So the probability for each classifier to make a correct prediction, that is to predict 1, is $p$. 

Now, substituting $X=S_n$, the mean and standard variance become $\mu=E\left[S_n\right]=n(p-(1-p))=n(2p-1)$ and $\sigma^2=E\left[(S_n-\mu)^2\right]=4np(1-p)$. This can easily be seen if we consider each independent prediction as a linear transformation of the form $2b-1$ of a Bernoulli variable $b \in \{0,1\}$.

The Chebychev's inequality is correct for any positive $\epsilon$, specifically for $\epsilon=\mu$. Therefore, applying the inequality to the case on $S_n$ one get:
\begin{equation}
   \Pr(|S_n-\mu |<\mu )> 1-{\frac {\sigma ^{2}}{\mu ^{2}}}=1-\frac{4p(1-p)}{n(2p-1)^2}=1-\alpha/n
\end{equation}
where $\alpha=\frac{4p(1-p)}{(2p-1)^2}$.

Since $\Pr(|S_n-\mu |<\mu )=\Pr(S_n>0 \cap S_n < 2 \mu)<\Pr(S_n>0)$, the latter reduces to:
\begin{equation}
1- \alpha /n < \Pr(S_n>0) <1
\end{equation}

Since $Pr(S_n)$ is a sequence bounded from both sides by values which both converge to 1, we can apply the squeeze theorem to prove that $\lim_{n \rightarrow \infty}\Pr(S_n>0)=1$, thus concluding the proof.
\end{proof}

\subsubsection{Proof Using the Recursion Equation}
The proof in Section \ref{subsec:chebychev} is sufficient. Yet, for the sake of completeness, we would to also demonstrate how the recursion formula developed in Section \ref{sec:monotonic} can be used to prove the Asymptotic behaviour.

\begin{proof}
As before, without loss of generality, assume that the true label is $y=1$. Thus, the probability of a majority rule to be correct in the limit of infinite voters can be expressed as follows:
\begin{equation}
\lim_{n \rightarrow \infty}\Pr(\hat{y}_{\text{majority}}^n=1)=\lim_{n \rightarrow \infty}\Pr(S_n>0) = \Pr(S_1>0) + \sum_{n=1}^{\infty}\Pr(S_{n+2}>0)-\Pr(S_n>0) 
\end{equation}.

Using the fact that $\Pr(S_1>0)=p$ and Eq. \ref{eq:recursive_simplified}, the above can be rewritten as:
\begin{equation}
\label{eq:cor_eq1}
   \lim_{n \rightarrow \infty}\Pr(\hat{y}_{\text{majority}}^n=1) = p + \sum_{m=0}^{\infty} {\binom{2m+1}{m} p^{m+1}(1-p)^{m+1}(2p-1)} 
\end{equation}

In the next Lemma \ref{lem:convto1} we show that the sum in the right hand side of Eq. \ref{eq:cor_eq1} equals to $1-p$.

\begin{lemma}
\label{lem:convto1}
For any positive integer $m$ and $1/2 < p \le 1$, the following holds:
\begin{equation}
\label{eq:recursive_sum}
\sum_{m=0}^\infty \binom{2m+1}{m} p^{m+1}(1-p)^{m+1}(2p-1) = 1-p
\end{equation}
\end{lemma}
\begin{proof}
Observe the following identify \cite{grahamconcrete}:

\begin{equation}
\label{eq:bionmial_sum}
\sum_{m=0}^{\infty}{\binom{2m+s}{m}x^m}=\frac{2^s}{(\sqrt{1-4x}+1)^s\sqrt{1-4x}}
\end{equation}

Assigning $s=1$ and $x=p(1-p)$ to Eq. \ref{eq:bionmial_sum} one gets:
\begin{equation}
\label{eq:binomial_sum_s1}
\sum_{m=0}^{\infty}{\binom{2m+1}{m}p^m(1-p)^m}=\frac{2}{(\sqrt{1-4p(1-p)}+1)\sqrt{1-4p(1-p)}}
\end{equation}

Note that $1-4p(1-p) = (2p-1)^2$, so Eq. \ref{eq:binomial_sum_s1} becomes:
\begin{equation}
\label{eq:binom_reduced}
\sum_{m=0}^{\infty}{\binom{2m+1}{m}p^m(1-p)^m}=\frac{1}{p(2p-1)}
\end{equation}

Plugging equation Eq. \ref{eq:binom_reduced} to Eq. \ref{eq:recursive_sum} we prove the Lemma \ref{lem:convto1}:
\begin{equation}
(2p-1)p(1-p)\sum_{m=0}^{\infty}{\binom{2m+1}{m}p^m(1-p)^m}= 1-p
\end{equation}
\end{proof} 
The above Lemma \ref{lem:convto1} concludes the proof, as:
\begin{equation}
   \lim_{n \rightarrow \infty}\Pr(\hat{y}_{\text{majority}}^n=1) = p + 1-p = 1
\end{equation}
\end{proof} 

\section{Summary}
In this short paper we have provided simple and compact derivations of the Condorcet's theorem. We have started by showing its optimality in terms of misclassification error. The we showed that the probability of the majority decision to be correct grows with the population size (non-asymptotic monotonic part). Last, we showed that this probability converges to 1 as the population size grows to infinity (asymptotic part). The latter was proved in two ways, one which was based on the recursion equation, and another based on Chebychev's inequality.

\section*{Acknowledgements}
Aryeh Kontorovich of the Computer Science Department
in the Ben-Gurion University for fruitful discussion.
\bibliographystyle{unsrt}  
\bibliography{references} 

\begin{thebibliography}{10}

\bibitem{condorcet1785essay}
Marquis~de Condorcet.
\newblock Essay on the application of analysis to the probability of majority
  decisions.
\newblock {\em Paris: Imprimerie Royale}, 1785.

\bibitem{ladha1992condorcet}
Krishna~K Ladha.
\newblock The condorcet jury theorem, free speech, and correlated votes.
\newblock {\em American Journal of Political Science}, pages 617--634, 1992.

\bibitem{freund1995desicion}
Yoav Freund and Robert~E Schapire.
\newblock A desicion-theoretic generalization of on-line learning and an
  application to boosting.
\newblock In {\em European conference on computational learning theory}, pages
  23--37. Springer, 1995.

\bibitem{friedman2001elements}
Jerome Friedman, Trevor Hastie, and Robert Tibshirani.
\newblock {\em The elements of statistical learning}, volume~1.
\newblock Springer series in statistics New York, 2001.

\bibitem{dietterich2000ensemble}
Thomas~G Dietterich.
\newblock Ensemble methods in machine learning.
\newblock In {\em International workshop on multiple classifier systems}, pages
  1--15. Springer, 2000.

\bibitem{kittler1998combining}
Josef Kittler, Mohamad Hatef, Robert~PW Duin, and Jiri Matas.
\newblock On combining classifiers.
\newblock {\em IEEE transactions on pattern analysis and machine intelligence},
  20(3):226--239, 1998.

\bibitem{kittler2003multiple}
Josef Kittler and Fabio Roli.
\newblock {\em Multiple Classifier Systems: First International Workshop, MCS
  2000 Cagliari, Italy, June 21-23, 2000 Proceedings}.
\newblock Springer, 2003.

\bibitem{moore1956reliable}
Edward~F Moore and Claude~E Shannon.
\newblock Reliable circuits using less reliable relays.
\newblock {\em Journal of the Franklin Institute}, 262(3):191--208, 1956.

\bibitem{lam1997application}
Louisa Lam and SY~Suen.
\newblock Application of majority voting to pattern recognition: an analysis of
  its behavior and performance.
\newblock {\em IEEE Transactions on Systems, Man, and Cybernetics-Part A:
  Systems and Humans}, 27(5):553--568, 1997.

\bibitem{ben2000nonasymptotic}
Ruth Ben-Yashar and Jacob Paroush.
\newblock A nonasymptotic condorcet jury theorem.
\newblock {\em Social Choice and Welfare}, 17(2):189--199, 2000.

\bibitem{berend2005monotonicity}
Daniel Berend and Luba Sapir.
\newblock Monotonicity in condorcet jury theorem.
\newblock {\em Social Choice and Welfare}, 24(1):83--92, 2005.

\bibitem{wiki}
Condorcet's jury theorem.
\newblock \url{https://en.wikipedia.org/wiki/Condorcet%27s_jury_theorem}.

\bibitem{stackexchange}
Condorcet's jury theorem - non asymptotic part.
\newblock
  \url{https://math.stackexchange.com/questions/2242127/condorcets-jury-theorem-non-asymptotic-part}.

\bibitem{devroye2013probabilistic}
Luc Devroye, L{\'a}szl{\'o} Gy{\"o}rfi, and G{\'a}bor Lugosi.
\newblock {\em A probabilistic theory of pattern recognition}, volume~31.
\newblock Springer Science \& Business Media, 2013.

\bibitem{grahamconcrete}
Ronald~L Graham, Donald~E Knuth, and Oren Patashnik.
\newblock Concrete mathematics: a foundation for computer science. 1994.
\newblock {\em Addison \& Wesley}, page 203 Eq. 5.68 and Eq. 5.72.

\end{thebibliography}

\end{document}